\documentclass[10pt,twocolumn]{article}

\usepackage[margin={1.25cm,1.25cm}]{geometry}

\usepackage{times}
\usepackage{authblk}

\usepackage{mathtools}
\usepackage{xspace}
\usepackage{soul}\setuldepth{x}
\usepackage{nameref}
\usepackage{subfigure}
\usepackage{amsmath,amssymb}
\usepackage{stmaryrd}
\usepackage{algorithm}
\usepackage{algorithmic}
\usepackage{comment}
\usepackage{graphicx} 
\usepackage{pgf}
\usepackage{tikz}
\usetikzlibrary{arrows,automata,shapes,trees,petri,topaths,positioning}
\usepackage{dot2texi}
\usepackage{listings}
\usepackage{amsmath}
\usepackage{amsthm}
\usepackage{amssymb}
\usepackage{proof}

\newcommand{\M}[1]{\mathit{#1}}
\newcommand{\fin}{\mbox{\small \sf fin}}

\newcommand{\meth}[1]{\llbracket #1 \rrbracket}
\newcommand{\tnet}[1]{[ #1 ]}

\tikzset{global scale/.style={
    scale=#1,
    every node/.style={scale=#1}
  }
}

\newcommand{\under}[1]{\ul{#1}}

\newcommand{\defterm}[1]{\under{\textit{#1}}}

\theoremstyle{definition}
\newtheorem{definition}{Definition} 
\newtheorem{proposition}{Proposition} 
\newtheorem{theorem}{Theorem} 

\newcommand{\namesmath}[1]{\mbox{\textit{#1}}}
\newcommand{\propername}[1]{\text{#1}}

\newcommand{\INDIGOLOG}{\propername{IndiGolog}}

\newcommand{\AgentSpeak}{\propername{AgentSpeak}}

\newcommand{\CAN}{\propername{CAN}}
\newcommand{\HTN}{HTN}
\newcommand{\REAP}{REAP}
\newcommand{\BDI}{BDI}

\newcommand{\CANPLAN}{\propername{CANPlan}}
\newcommand{\JACK}{\propername{JACK}}

\newcommand{\PRS}{\propername{PRS}}

\newcommand{\JSHOP}{\propername{JSHOP}}

\newcommand{\PROPICE}{\propername{Propice-Plan}}

\newcommand{\RETSINA}{\propername{RETSINA}}

\newcommand{\refn}{\namesmath{evo}}
\newcommand{\sol}{\namesmath{sol}}

\newcommand{\comp}{\namesmath{comp}}

\newcommand{\redd}{\namesmath{red}}

\newcommand{\Op}{\namesmath{Op}}

\newcommand{\Me}{\namesmath{Me}}

\newcommand{\move}{\namesmath{mv}}

\newcommand{\estabConn}{\namesmath{estabCon}}

\newcommand{\breakConn}{\namesmath{breakCon}}

\newcommand{\first}{\normalsize{\namesmath{fst}}}
\newcommand{\last}{\normalsize{\namesmath{lst}}}

\newcommand{\true}{\namesmath{true}}

\newcommand{\transmitData}{\namesmath{transDS}}
\newcommand{\sendData}{\namesmath{sendData}}

\newcommand{\extData}{\namesmath{tagData}}
\newcommand{\raw}{\namesmath{raw}}

\newcommand{\monitor}{\namesmath{procImg}}
\newcommand{\sendExtData}{\namesmath{sendTagData}}
\newcommand{\uploadData}{\namesmath{loadDS}}
\newcommand{\calibrate}{\namesmath{calib}}
\newcommand{\heat}{\namesmath{mvC}}
\newcommand{\navigate}{\namesmath{nav}}
\newcommand{\lander}{\namesmath{landr}}

\newcommand{\didExp}{\namesmath{didExp}}

\newcommand{\cali}{\namesmath{cal}}
\newcommand{\charge}{\namesmath{charge}}
\newcommand{\stormy}{\namesmath{lowBat}}
\newcommand{\storm}{\namesmath{isStormy}}
\newcommand{\stormact}{\namesmath{stormy}}
\newcommand{\connEst}{\namesmath{connEst}}
\newcommand{\camMoved}{\namesmath{camMoved}}
\newcommand{\nop}{\namesmath{nop}}

\newcommand{\DD}{\mbox{$\mathbb{D}$}}

\newcommand{\TT}{\mbox{$\mathcal{T}$}}

\newcommand{\TTT}{\mbox{${\mathcal{\overline{T}}}$}}

\newcommand{\II}{\mbox{$\mathcal{I}$}}

\newcommand{\tuple}[1]{\langle #1 \rangle}
\newcommand{\config}[1]{\M{\langle #1 \rangle}}

\def\planeaux!#1:#2<-#3!{\M{#1 \mbox{\rm:} #2\; \leftarrow #3}}
\def\planeVaux!#1[#2]:#3<-#4!{#1(\vec{#2}) \mbox{\rm:} #3\; \leftarrow #4}

\def\planeaux!#1<-#2!{{#1 \leftarrow #2}}

\newcommand{\Plan}{\mbox{\sf Plan}}

\begin{document}

\title{Addendum to ``HTN Acting: A Formalism and an Algorithm''}
\author{Lavindra de Silva}
\affil{Department of Engineering, \protect\\  University of Cambridge, UK \protect\\ Lavindra.deSilva@eng.cam.ac.uk}
\date{}

\maketitle

\begin{abstract}
Hierarchical Task Network (\HTN) planning is a practical and efficient approach to planning when the `standard operating procedures' for a domain are available. Like Belief-Desire-Intention (\BDI) agent reasoning, \HTN\ planning performs hierarchical and context-based refinement of goals into subgoals and basic actions.
However, while \HTN\ planners `lookahead' over the consequences of choosing one refinement over another, \BDI\ agents interleave refinement with acting.
There has been renewed interest in making \HTN\ planners behave more like \BDI\ agent systems, e.g. to have a unified representation for acting and planning. However, past work on the subject has remained informal or implementation-focused.
This paper is a formal account of `\HTN\ acting', which supports interleaved deliberation, acting, and failure recovery. 
We use the syntax of the most general \HTN\ planning formalism and build on its core semantics, and we provide an algorithm which combines our new formalism with the processing of exogenous events. We also study the properties of \HTN\ acting and its relation to \HTN\ planning.
\end{abstract}


\section{Introduction}

Hierarchical Task Network (\HTN) planning 
\cite{ErolHN:AMAI96,Nau:IJCAI99-SHOP,Nau:IJCAI01-SHOP2,GhallabNT:04-Planning} 
is a practical and efficient approach to
planning when the `standard operating procedures' for a domain 
are available. 
\HTN\ planning is similar to
Belief-Desire-Intention (\BDI) \cite{Rao:KR91,Rao:LNCS96-AgentSpeak,Hindriks1999,Winikoff.etal:KR02-CAN} agent
reasoning in that both approaches perform hierarchical 
and context-based refinement of goals into subgoals 
and basic actions 
\cite{SardinaDP:AAMAS06,SardinaPadgham:JAAMAS10}.
However, while \HTN\ planners `lookahead' 
over the consequences of choosing one refinement 
over another before suggesting an action, 
\BDI\ agents interleave refinement 
with acting in the environment. Thus, while the former 
approach can guarantee goal achievability (if there 
is no action failure or environmental interference), 
the latter approach is able to quickly respond to environmental 
changes and exogenous events, and recover from failure.
This paper presents a formal semantics that builds on 
the core \HTN\ semantics in order to enable such response 
and recovery.

One motivation for our work is a recent drive toward adapting 
the  languages and algorithms used in Automated 
Planning to build a framework for `refinement acting' 
\cite{Ghallab:2016}, i.e., deciding how to carry out 
a chosen recipe of action to achieve some objective, 
while dealing with environmental changes, events, and 
failures. To this end, \cite{Ghallab:2016} proposes the 
Refinement Acting Engine (RAE), an \HTN-like framework 
 with continual online processing and recipe repair 
in the case of runtime failure. A key consideration 
in the RAE is a unified hierarchical representation 
and a core semantics that suits the needs of both acting 
and lookahead.
We are also motivated by recent work \cite{deSilva:2017} 
which suggests that a fragment of the recipe language of \HTN\ 
planning does not have a direct (nor known) translation to the 
recipe languages of typical \BDI\ agent programming 
languages such as \AgentSpeak\ \cite{Rao:LNCS96-AgentSpeak} 
and \CAN\ \cite{Winikoff.etal:KR02-CAN}. For example, HTNs 
allow a flexible specification of how steps in a recipe 
should be interleaved, whereas steps in \CAN\ recipes must 
be sequential or interleaved in a `series-parallel' 
\cite{Valdes:1979} manner.

There have already been some efforts toward adapting \HTN\ 
planning systems to make them behave more like \BDI\ agent systems.
Perhaps the first of these efforts was the \RETSINA\ architecture 
\cite{Sycara:2003},
which used an \HTN\ language and semantics for representing 
recipes and refining tasks, but also interleaved task
refinement with acting in the environment. \RETSINA\ is an implemented 
architecture which has been used in a range of real-world 
applications.
In \cite{deSilvaP:AAI04}, the \JSHOP\ \cite{Nau:IJCAI99-SHOP} 
\HTN\ planner is modified in two ways: \textit{(i)} 
to execute a solution (comprising a sequence of actions) found via lookahead, and then 
re-plan if the solution is no longer viable 
in the real world (due to a change in the environment), and 
\textit{(ii)} to 
immediately execute the chosen refinement for a task, instead 
of first performing lookahead 
to check whether the 
refinement will accomplish the task. The latter 
modification made \JSHOP\ as effective as the 
industry-strength 
\JACK\ \BDI\ agent framework \cite{Winikoff2005}, in terms of 
responsiveness to environmental changes. 

However, both \RETSINA\ and the \JSHOP\ variant lack a formalism,
making it difficult to study the properties (e.g. correctness) of 
their semantics, and to compare them to other 
similar systems. The same applies to the algorithms and abstract syntax 
of the RAE framework, which are presented only in pseudocode.

There is also some work on making \BDI-like agent systems 
behave more like \HTN\ planning systems. In particular, 
both the \REAP\ algorithm in \cite{Ghallab:2016} and the 
\CANPLAN\ \cite{SardinaDP:AAMAS06,SardinaPadgham:JAAMAS10} \BDI\ agent programming 
language (and its extensions such as \cite{Bauters:14,deSilva:2007}) can make informed decisions about  refinement choices
by using a lookahead capability.
Similarly, there are 
 agent programming
languages and systems that support some form of planning  
(though not \HTN-style planning) 
\cite{MeneguzziSilva:KER16}, such as the \PRS\ \cite{Georgeff:IJCAI89-PRS}
based 
\PROPICE\ \cite{DespouysI:ECP99} system and the situation-calculus 
based \INDIGOLOG\ \cite{DGLLS:APLBOOK09-IndiGolog} system.
Finally, there are also some interesting extensions
to \HTN\ and \HTN-like planning 
\cite{Geier:2011,Alford:BA15,Kambhampati:AAAI98,Xiao:17,Shivashankar.etal:IJCAI13-GoDeL,Alford:16}, 
e.g. approaches that combine classical and \HTN\ planning.
In contrast, our work is not concerned with lookahead 
or planning, but with adapting the \HTN\ 
planning semantics to enable \BDI-style behaviour.

Thus, our contribution is a formal account of \HTN\ 
\emph{acting}, which supports interleaved deliberation, 
acting, and recovery from failure, e.g. due to environmental 
changes. To this end, we use the syntax of the
most general \HTN\ planning formalism \cite{ErolHN:AMAI96,Erol:1994}, and we build on its core semantics by developing three main definitions: execution via reduction, 
action, and replacement. We then provide an algorithm for 
\HTN\ acting which combines our new formalism with the 
 processing of exogenous events. We also study the properties of \HTN\ acting, particularly in relation to \HTN\ planning.

\newcommand{\res}{\mbox{\small \sf res}}
\newcommand{\resu}{\mbox{\small \sf result}}
\newcommand{\rel}{\mbox{\small \sf bef}}
\newcommand{\exec}{\mbox{\small \sf exec}}
\newcommand{\fst}{\mbox{\small \sf primary}}
\newcommand{\suc}{\mbox{\small \sf succ}}
\newcommand{\pre}{\mbox{\small \sf pre}}
\newcommand{\fail}{\mbox{\small \sf fail}}
\newcommand{\add}{\mbox{\small \sf add}}
\newcommand{\del}{\mbox{\small \sf del}}
\newcommand{\relv}{\mbox{\small \sf rel}}
\newcommand{\rep}{\mbox{\small \sf rep}}
\newcommand{\upd}{\mbox{\small \sf upd}}
\renewcommand{\redd}{\mbox{\small \sf red}}
\newcommand{\lab}{\mbox{\small \sf lab}}
\newcommand{\smallest}{\mbox{\small \sf smallest}}
\newcommand{\blocked}{\mbox{\small \sf blocked}}
\newcommand{\act}{\mbox{\small \sf act}}
\renewcommand{\sol}{\mbox{\small \sf sol}}
\renewcommand{\comp}{\mbox{\small \sf comp}}
\renewcommand{\refn}{\mbox{\small \sf evo}}
\newcommand{\topp}{\mbox{\small \sf top}}

\renewcommand{\vec}[1]{\mathbf{#1}}
\renewcommand{\bar}[1]{\underline{#1}}

\renewcommand{\algorithmicrequire}{\textbf{Input:}}
\newcommand{\leftarrowtwo}{:=}
\renewcommand{\algorithmicensure}{\textbf{Output:}}
\newcommand{\Exec}{\mbox{\textsf{Sense-Reason-Act}}}

\section{Background: \HTN\ Planning}
\label{sec:back-htn}

In this section we provide the necessary background
material on \HTN\ planning. 
Some definitions are given only informally; we refer the 
reader to \cite{ErolHN:AMAI96,Erol:1994} for the formal 
definitions.

An \HTN\ \defterm{planning problem} is a tuple
$\tuple{d,\II,\DD}$ comprising a \emph{task network} 
$d$, an initial \defterm{state} $\II$, which is a set of ground
atoms, and a 
\defterm{domain} $\DD = \tuple{\Op,\Me}$, where $\Me$ 
is a set of reduction \emph{methods} and $\Op$ is a 
set of STRIPS-like operators. \HTN\ planning involves
iteratively decomposing/reducing the `abstract tasks' occurring
in $d$ and the resulting task networks by using methods 
in $\Me$, until only STRIPS-like actions remain that 
can be ordered and executed from $\II$ relative to $\Op$.

A \defterm{task network} $d$ is a couple
$\tnet{S,\phi}$, where $\phi$ is a \emph{constraint 
formula}, and $S$ is a non-empty set of \emph{labelled 
tasks}, i.e., constructs of the form $(n:t)$;  
element $n$ is a \defterm{task label}, which is a 0-ary 
task-label symbol (in FOL) that is unique in $d$ and $\DD$, 
and $t$ is a \emph{non-primitive} or \emph{primitive} \defterm{task}, which is an n-ary 
task symbol whose arguments are function-free terms. 
The \defterm{constraint formula} $\phi$ is a
Boolean formula built from negation, disjunction,
conjunction, and constraints, each of which is either:
%
an \defterm{ordering constraint} of the form $(n \prec n')$,
which requires the task (corresponding to label) $n$ to 
precede task $n'$;
a \defterm{before} (resp. an \defterm{after}) \defterm{state-constraint} 
of the form $(l,n)$ (resp. $(n,l)$), which requires literal 
$l$ to hold in the state just before (resp. after) doing
 $n$; 
a \defterm{between state-constraint} of the form $(n,l,n')$, 
which requires  $l$ to hold in all states between 
doing  $n$ and $n'$; or 
%
%
a \defterm{variable binding constraint} of the form $(o = o')$, which requires
$o$ and $o'$ to be equal, each of which
is a variable or constant. 
We ignore variable binding constraints as they can be specified as state-constraints, using the binary logical symbol `='.

Instead of specifying a task label, 
a constraint may also refer, using
expression $\first[S]$ or $\last[S]$,
to the action that is eventually ordered to occur first or last (respectively) among those that are yielded by the set of task labels $S$. While these expressions can be `inserted' into a constraint when a task is reduced, we assume that they do not occur in methods.

A primitive task, or \emph{action}, $t$, has exactly one \emph{relevant operator} in $\Op$, i.e., one operator associated with a primitive task $t'$ that has the same task symbol and arity as $t$; any variable appearing in the operator  also appears in $t'$ and its precondition. Given a primitive task $t$, we denote its precondition, add-list and delete-list relative to $\Op$ as $\pre(t,\Op), \add(t,\Op)$ and $\del(t,\Op)$, respectively. 
A non-primitive task can have one or more 
\emph{relevant  methods} in $\Me$. A 
\defterm{method} is a couple $\meth{t(\vec{v}), d}$, where 
$t(\vec{v})$ is a non-primitive task, the arguments $\vec{v}$ 
are distinct variables,\footnote{While \cite{ErolHN:AMAI96}
does allow this vector to contain constants, we instead specify 
such binding requirements in the constraint formula.}
and $d$ is a task network.

Given an \HTN\ planning problem 
$\tuple{d = \tnet{S_d,\phi_d},\II,\tuple{\Op,\Me}}$,  
the core  planning steps involve selecting a 
relevant method $m=\meth{t_m,d_m} \in \Me$ 
for some non-primitive task $(n:t) \in S_d$ and 
then reducing the task to yield a 
`less abstract' task network. 
Reducing $(n:t)$ 
with $m$ involves replacing $(n:t)$ with the tasks in 
$S_m$ (where $d_m = \tnet{S_m,\phi_m}$) and updating 
$\phi_d$, e.g. to include the constraints in $\phi_m$;
formal definitions for method relevance and reduction are 
given in Section \ref{sec:prelim}. The set of reductions 
of $d$ is denoted  $\redd^*(d,\tuple{\Op,\Me})$.

If all  non-primitive tasks in the initial 
and subsequent task networks have been reduced,
a \emph{completion}  is obtained
from the resulting `primitive' task network.
Informally, $\sigma$ is a completion of a primitive task
network $d=\tnet{S,\phi}$ at a state $\II$, denoted $\sigma \in comp(d,\II,\DD)$, if $\sigma$ is a 
total ordering of a ground instance of $d$ 
that satisfies $\phi$; 
if $d$ mentions a non-primitive task,  
then $comp(d,\II,\DD) = \emptyset$.

Finally, the set of all \HTN\ 
\defterm{solutions} is defined as
$sol(d,\II,\DD) = \bigcup_{n < \omega} 
sol_n(d,\II,\DD)$, where $sol_n(d,\II,\DD)$ 
is defined inductively as 
{
\begin{eqnarray*}
sol_1(d,\II,\DD)        & = & comp(d,\II,\DD), \\
sol_{n+1}(d,\II,\DD) 	& = &
                       	sol_n(d,\II,\DD) \cup \hspace*{-0.5cm}
                       	\bigcup_{d' \in {\scriptsize \textsf{red}^*}(d, \scriptsize \DD)} \hspace*{-0.5cm} sol_n(d',\II,\DD).
\end{eqnarray*}
}

In words, the \HTN\ solutions for a given
planning problem is the set of all completions 
of all primitive task networks that can be 
obtained via zero or more reductions of the 
initial task network.

\subsection*{A Running Example}

Let us consider the example of
a rover agent exploring the surface of mars. A part 
of the rover's HTN domain is illustrated in Figure 
\ref{fig:ex} (with  braces omitted in $\first[]$ 
and $\last[]$ expressions). The top-level non-primitive task 
is to transfer, to the lander, previously gathered
soil analysis data from a location $X$,
and if possible to also deliver
the soil sample for further analysis inside the lander.
 
The top-level task is achieved using either method 
$m_1$ or $m_2$, both of which require the data and 
sample from  $X$ to be available (i.e., for 
$\didExp(X)$ to hold).
If the rover is low on battery charge ($\stormy$), 
$m_1$ is used. This transmits the soil data but it does 
not deliver the soil sample, which may result in losing
it if it is later discarded to make room for other samples.
Method $m_1$ prescribes establishing radio communication
with the lander, sending it the data by first including 
metadata, and then breaking 
the connection,  while checking continuously that the 
connection is not lost between the first and last 
tasks (including those of $m_3$).
If the rover is not low on battery charge, 
$m_2$ is used to achieve the top-level task; $m_2$ 
prescribes navigating to a lander $L$ and then 
uploading and depositing the soil data and sample, 
respectively.

\begin{figure}
\resizebox{0.912\columnwidth}{!}{
\begin{tabular}{c}
\begin{tikzpicture}[]

  \tikzstyle{every node}=[draw=black,thick]
  \tikzstyle{edge from parent}=[black,thick,draw,->]

 \tikzstyle{ptask}=[dotted]
  \tikzstyle{method}=[rounded corners, fill=black!15]
  \tikzstyle{photon} = [snake=snake, draw=red]
  \tikzstyle{cross} = [cross out,draw=blue,very thick,minimum size=0.8cm]
  \tikzstyle{lab} = [draw, ellipse, node distance=3cm, minimum height=2em]

  
  \node[label={[xshift=-3mm]below:\textbf{or}}] (top) {$\transmitData(X)$}
                    child[sibling distance=6.7cm,level distance=1.2cm] {node[method] {$m_1$}            
                        child[sibling distance=2.3cm,level distance=0.85cm]{node[ptask] (l1) {1:$\estabConn$}}
                        child[sibling distance=2.4cm,level distance=0.85cm]{node {2:$\sendData(X)$}
                                child[level distance=0.85cm]{node[method] {$m_3$}            
                                    child[level distance=0.85cm,sibling distance=2.6cm] {node[ptask] {4:$\extData(X)$}}
                                    child[level distance=0.85cm,sibling distance=2.8cm] {node[ptask] {5:$\sendExtData(X)$}}
                                }
                        }
                        child[sibling distance=2.3cm,level distance=0.8cm] {node[ptask] {3:$\breakConn$}}
                    }
                    child[level distance=3.85cm,sibling distance=0cm]{node[method] {$m_2$}            
                        child[sibling distance=2.55cm,level distance=0.8cm]{node[label={[xshift=-0mm]below:\textbf{or}}]{6:$\navigate(L)$}
					        child[sibling distance=4cm,level distance=0.8cm] {node[method] {$m_4$}
					    	    child[sibling distance=1.3cm,level distance=0.8cm] {node(cal)[ptask] {8:$\calibrate$}}
					    	    child[sibling distance=1.4cm,level distance=0.8cm] {node[ptask] {9:$\heat$}}
					    	    child[sibling distance=1.5cm,level distance=0.8cm] {node[ptask] {10:$\move(L)$}}
					        }
                            child[sibling distance=4cm,level distance=0.8cm] {node[method] {$m_5$}
                                child[sibling distance=1.6cm,level distance=0.8cm] {node[ptask] {11:$\heat$}}
                                child[sibling distance=1.6cm,level distance=0.8cm] {node[ptask] {12:$\move(L)$}}
                            }  
				        }
				        child[sibling distance=2.2cm,level distance=0.8cm] {node[ptask] {7:$\uploadData(X)$}}
				    };
%
%
%
 \node [above left of=top,shift={(-5cm,-0.7cm)}]{Task};
 \node [method,above left of=top,shift={(-5cm,-1.25cm)}]{Method};
 \node [ptask,above left of=top,shift={(-5cm,-1.8cm)}]{Action};
 \node [below left of=cal,shift={(3.8cm,-1.8cm)}]
       [text centered,rounded corners]
        { \begin{tabular}{ccc}
 		$\phi_1$                & $\phi_2$                          & $\phi_3$    \\ \hline     
        $(\stormy,1)$        & $(\lander(L),6)$                  & $(4 \prec 5)$           \\          
        $(1, \connEst, 3)$ &   $(\didExp(X),6)$                &                           \\          
        $(\didExp(X),1)$        &   $(6 \prec 7)$                &                           \\          
        $(1 \prec 2), (2 \prec 3)$           &         &                           \\    \hline     
        $\phi_4$                          & $\phi_5$              &\\ \hline
        $(\neg \cali,\first[8,9])$        & $(\cali,11)$          &\\
        $(\neg \stormy,\first[8,9])$      & $(\neg \stormy,11)$   &\\
        $(\last[8,9] \prec 10)$           & $(11 \prec 12)$       &\\
 	  \end{tabular}
  	};

\end{tikzpicture}
\end{tabular}
}
\caption{A partial  domain for a simple rover. 
The  tasks and methods are shown at the top, and the constraint formulas of methods are shown in the table. 
Each $m_i$ is of the form $\meth{t_i,d_i=\tnet{S_i,\phi_i}}$.
We use expressions $\first[]$ and $\last[]$ only for readability; 
it is straightforward to replace their associated constraints with those
that do not contain such expressions.
}\label{fig:ex}
\end{figure}
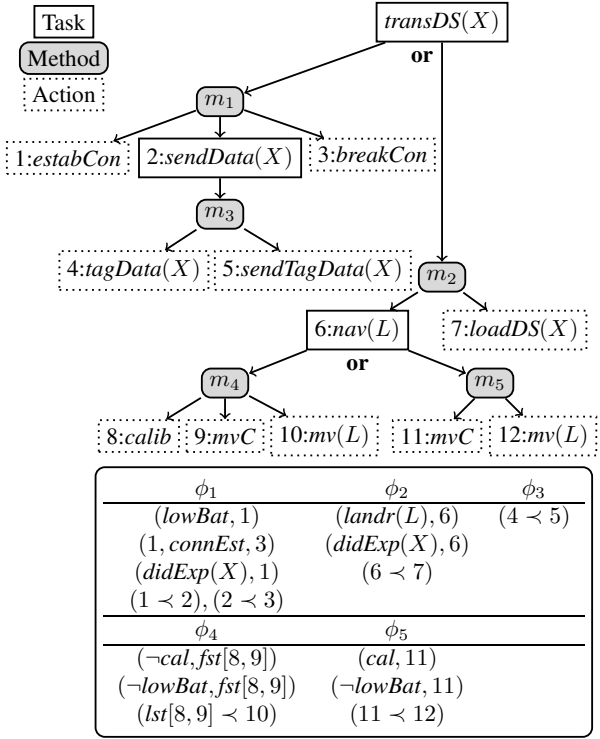

Navigation is performed  using $m_4$ or $m_5$. Method 
$m_4$ prescribes calibrating the onboard instruments, moving 
the cameras to point straight (which asserts $\camMoved$), 
and moving to the lander; while the first two actions can 
happen in any order, the third must happen last. 
The method requires that the instruments are not 
currently calibrated ($\neg \cali$) and the battery charge is not 
low. 
Method $m_5$ is similar except that it is used only
if the instruments are already calibrated, for example
due to a recent calibration to achieve another task.

Action $\move$ requires $\neg \stormy$ 
to hold, and it consumes
a significant amount of  charge, i.e., it asserts 
$\stormy$.\footnote{For simplicity, we assume 
`low charge' is less than or equal to 50\% of 
the maximum charge,
and an action requiring a `significant' 
amount of charge consumes 50\%. We also consider it
unsafe for the charge to reach 0\%.}
Action $\monitor$ (not shown) requires $\raw$ and 
$\neg \stormy$ to hold and asserts $\neg \raw$ and $\stormy$;
the action processes and compresses new 
raw images (if any exist, i.e., $\raw$ holds) 
of the martian surface that were taken by 
the cameras. Doing $\monitor$ infrequently 
may result in losing older images, if they are 
overwritten to make space on the storage 
device.\footnote{We assume that 
delivering a soil sample to the lander and processing images 
before they are overwritten have equal importance.}
The other actions consume a negligible amount 
of charge, and action $\charge$ (not shown) makes
the battery fully charged.

\section{Preliminaries and Assumptions}
\label{sec:prelim}

In this section we formally define the notion of reduction, 
and we state the remaining assumptions. 

First, we separate  the notion of method relevance from the notion of reduction in \cite{ErolHN:AMAI96}. In what follows, we use the standard notion 
of \emph{substitution} \cite{Lloyd:BOOK87-LogicProgramming}, and of  \emph{applying} a substitution 
$\theta$ to an expression $E$, which we
denote by $E\theta$.

\begin{definition}[Relevant Method]\label{def:relmeth}
Let $\DD=\tuple{\Op,\Me}$ be a domain, $t$ a non-primitive
task, and $\meth{t',d} \in \Me$ a method. If $t = t'\theta$
for some substitution $\theta$, then
$d\theta$ is a \defterm{relevant}  method-body for $t$ 
relative to $\DD$.\footnote{All variables and task 
labels in $d\theta$ must be renamed with variables
and task labels that do not appear anywhere else.}
The set of all such 
method-bodies is denoted by $\relv(t,\DD)$.
\end{definition}

In the definition of reduction below, and in the rest 
of the paper, we denote by $\lab(S)$ the set of all 
task labels appearing in a given set of labelled tasks $S$.

\begin{definition}[Reduction (adapted from 
\cite{ErolHN:AMAI96})] \label{def:red}
%
Let $d = \tnet{\{(n:t)\} \cup S, \phi}$,
with $(n:t) \not\in S$, be a task network
and $t$ a non-primitive task, and let
$d' = \tnet{S',\phi'} \in \relv(t,\DD)$.
%
The \defterm{reduction} of $n$ in $d$
with $d'$, denoted $\redd(d,n,d')$, is the task network
$\tnet{S \cup S', \phi' \land \psi}$, 
where $\psi$ is obtained from $\phi$ with the following
modifications:
\begin{itemize}

\item replace $(n \prec n_j)$ with $(\last[\lab(S')]
      \prec n_j),$ as $n_j$ must come after every task in $n$'s
      decomposition;

\item replace $(n_j \prec n)$ with $(n_j \prec \first[\lab(S')])$;

\item replace $(l,n)$ with $(l, \first[\lab(S')]),$
      as $l$ must be true immediately before the first task
      in $n$'s decomposition;

\item replace $(n,l)$ with $(\last[\lab(S')],l),$ as $l$
      must be true immediately after the last task in $n$'s
      decomposition;

\item replace $(n,l,n_j)$ with $(\last[\lab(S')],l,n_j)$; 

\item replace $(n_j,l,n)$ with $(n_j,l,\first[\lab(S')])$; and

\item everywhere that $n$ appears in $\phi$ in a $\first[]$
      or a $\last[]$ expression, replace it with $\lab(S')$.
%
\end{itemize}
\end{definition}

For example, consider  task network $d=\tnet{S,\phi}$,
where $S=\{(A:\transmitData(loc1)),(B:\charge)\}$
and  $\phi = (A \prec B)$. Observe that method
$m_2$ in Figure \ref{fig:ex} is 
$\meth{\transmitData(X),d_2=\tnet{S_2,\phi_2}}$, where $S_2=\{(6:\navigate(L)), (7:\uploadData(X))\}$ and $\phi_2 = (\lander(L),6) \land (\didExp(X),6) \land
(6 \prec 7)$. Then, 
$\redd(d,A,d_2)$ is  task
network $\tnet{S',\phi'}$ where $S' =
\{
(6:\navigate(L)), (7:\uploadData(loc1)), (B:\charge)
\}$
and  $\phi'$ is the conjunction of $\phi_2$
and $\phi$ updated to account for the
reduction, i.e.,
$\phi' = (\lander(L),6) \land (\didExp(loc1),6) \land
(6 \prec 7) \land (\last[6,7] \prec B)$. 

In the rest of the paper, we ignore the $\charge$ 
task, and when we need to refer to a labelled 
task we simply use its task label if the corresponding
task is obvious; e.g. we would  represent $S'$ above as 
$\{6,(7:\uploadData(loc1)),B\}$.

The remaining assumptions that we make are the
following. First,
without loss of generality \cite{deSilva:2017}, we assume that \HTN\
domains  are \emph{conjunctive}, i.e., they do not mention 
constraint formulas that specify a disjunction of
elements.
Thus, we sometimes treat a constraint formula as a set (of
possibly negated constraints).

\begin{definition}[Conjunctive HTNs \cite{deSilva:2017}]
\label{def:norm} A task network $\tnet{S,\phi}$ 
is \defterm{conjunctive} if its constraint formula $\phi$ 
is a conjunction of possibly negated constraints. A 
\defterm{domain} $\tuple{\Op,\Me}$ is \defterm{conjunctive} 
if the task network $d$ in every method $\meth{t,d} \in \Me$ 
is conjunctive.
\end{definition}

Second, to distinguish between reductions that are being
pursued at different levels of abstraction, we assume  
a reduction produces at least two tasks, i.e., any method 
$\meth{t,\tnet{S,\phi}} \in \Me$ is such that $|S|>1$. This
can be achieved using `no-op' actions, denoted $\nop$, if necessary,
which have `empty' preconditions and effects. 

Third,
for any method $\meth{t,\tnet{S,\phi}} \in \Me$, 
there exists a (possibly `no-op') task $(n:t) \in S$ such that
$(n' \prec n) \in \phi$ for
any $n' \in \lab(S) \setminus \{n\}$, and
%
$(n,l) \not\in \phi$ for any $l$. This will ensure that 
all the after state-constraints in $\phi$ are evaluated
by our semantics.

Finally, we assume that the user does not specify inconsistent
ordering constraints in a method's constraint formula, e.g.
the constraints $(1 \prec 2), (2 \prec 3),$ and 
$(3 \prec 1)$. Formally, 
let $\phi^*$ denote the transitive closure of 
a constraint formula $\phi$, i.e., the one
that is obtained from $\phi$ by adding the
constraint $(n_1 \prec n_{i+1})$ whenever 
$(n_1 \prec n_2), (n_2 \prec n_3), \ldots, 
(n_i \prec n_{i+1}) \in \phi$ holds for 
some $i>1$. Then, for
any method $\meth{t,\tnet{S,\phi}} \in \Me$,
there does not exist a pair 
$(n \prec n'), (n' \prec n) \in \phi^*$ nor
$(n \prec n'), \neg (n \prec n') \in \phi^*$.

\section{A Formalism for HTN Acting}

We now develop a formalism for \HTN\ acting by defining, 
in particular, three notions of execution: via \emph{reduction}, 
\emph{action}, and \emph{replacement}. The first notion is based 
on task reduction; the second notion defines what it means to execute an 
action in the \HTN\ setting, in particular, the gathering and evaluating of
 constraints  relevant to the action; and the third notion
represents failure handling, i.e., the replacement of `blocked' 
tasks by alternative ones.

We only allow a task occurring in a task
network to be executed via action or reduction
if it is a \emph{primary} task in the  network, 
i.e., there are no other tasks 
that \emph{must} precede it.
Formally, given a task network $d=\tnet{S,\phi}$,
we first define the following sets of tasks: 
\[
\begin{array}{ll}
S_1 = \{(n:t) \in S \ \ \mid & (x \prec x') \in \phi, n$ occurs in $x'\},\\ 
S_2 =  \{(n:t) \in S \ \ \mid & t\ \text{is an action
and either}\ \neg (n \prec x) \in \phi\ \text{or}\\ & \neg (\last[\{n\}] \prec x) \in \phi\}.
\end{array}
\]
That is, $S_1$ 
and $S_2$ contain the tasks that \emph{cannot}
be primary ones; the above action $n$ occurring in a 
negated ordering constraint cannot be a primary task because one
or more tasks (represented by $x$ above) must precede 
$n$.\footnote{This is provided none of the actions associated with $x$ have
already been executed. As we show later, in 
our semantics, such an execution will result in the (then `realised') constraint being 
removed.}
Then, we define the set of \defterm{primary} 
tasks of task network $d$ as 
$\fst(d) = S \setminus (S_1 \cup S_2)$.
For example, given task network $d_1$ in method $m_1$ 
in Figure \ref{fig:ex}, $\fst(d_1)=\{1\}$, and given
task network $d_4$ in method $m_4$, $\fst(d_4)=\{8,9\}$.

We can now define our first notion, an
\emph{execution via reduction} of a task network, as the
reduction of an arbitrary primary non-primitive
task via a relevant method. To enable trying alternative
reductions for the task if the one that was selected fails 
or is not applicable, we maintain the set of all relevant
methods for the task, and update the set as  alternative 
methods are tried. We use the term \emph{reduction couple}
to refer to a couple comprising two sets: \textit{(i)} 
the set representing the reductions being pursued for a 
task (and its subtasks), and \textit{(ii)} 
the set of current alternative method-bodies
for the task.
We use $R$ to denote the set of  reduction couples 
corresponding to the tasks reduced so far, where each 
couple is of the form $\tuple{S,D}$, with $S$ being
a set of labelled tasks, and $D$ a set of task networks.
While the initial value
of $R$ and 
how it can `evolve' will be made concrete via formal definitions, we shall for now illustrate these
with an example.

Let us consider the task network 
$\tnet{S, \phi=\true}$, where the set
$S=\{(A:\transmitData(loc1)),(B:\monitor)\}$;
the initial state
$\II = 
\{
\raw,
\cali, \didExp(loc1),\lander(lan1)
\}$;
the `initial' set of reduction couples $R=
\{ \tuple{ S, \emptyset} \}$; and the 
domain $\DD$ is as depicted in Figure \ref{fig:ex}. 
An execution via reduction of the task network
from $\II$ relative to $R$ and $\DD$ 
is the tuple 
$\tuple{\tnet{S',\phi'},\II,R',\DD}$, where 
$S' = 
\{
6,7,B
\}$,
formula $\phi'$ is $\phi_2$ in Figure \ref{fig:ex}
with variable $X$ substituted with $loc1$,
and the resulting set of reduction couples 
$R' = 
\{ 
\tuple{S',\emptyset},
\tuple{\{6,7\}, \{d_1\}}
\}$, where $d_1$ is the alternative method-body for $A$.
Moreover, an execution via reduction of $\tnet{S',\phi'}$ 
is  the tuple
$\tuple{\tnet{S'',\phi''},\II,R'',\DD}$, where 
$S'' = S''' \cup \{7,B\}$,
set
$S''' = \{11, 12\}$, formula $\phi''$ 
is the conjunction of $\phi_5$ and $\phi'$ 
updated to account for the reduction, and set
 $R''= 
\{
\tuple{S'',\emptyset},
\tuple{S''' \cup \{7\}, \{d_1\}},
\tuple{S''',\{d_4\}}
\}$. 

We call a 4-tuple of the form $\tuple{d,\II,R,\DD}$, 
as in the example above, a \defterm{configuration}. 
(For brevity, we omit the fifth element $\theta$,
representing the substitutions applied so far to 
variables appearing in $d$.)
Formally, we define an execution via reduction as follows.

\begin{definition}[Execution via Reduction]\label{def:exec-red}
Let $\DD$ be a domain; $\II$ a state; $d$ 
a task network with a non-primitive task
$(n:t) \in \fst(d)$; $R$ a set of reduction 
couples; 
$d_n = \tnet{S_n,\phi_n} \in D$ a 
method-body, with $D = \relv(t,\DD)$; 
and couple $r=\tuple{S_n, D \setminus \{d_n\}}$.
An \defterm{execution} via \defterm{reduction}
of $d$ from $\II$ relative to $R$ and $\DD$ 
is the configuration
$\tuple{\redd(d,n,d_n),\II,R' \cup \{r\},\DD}$, 
where $R'$ is $R$ with any occurrence of $(n:t)$ 
replaced by the elements in set $S_n$.
\end{definition}

We now define the second kind of execution: performing 
an action. In order to execute a (primary) action, it 
must be \emph{applicable}, i.e., its precondition and any 
constraints that are \emph{relevant} to the action must 
hold in the current state. Such constraints could have been 
(directly) specified on the action, `inherited' from one or more of 
the action's `ancestors', or `propagated' from 
an earlier action. We first define the 
notion of a relevant constraint; we ignore negated 
between state-constraints for brevity.%
\footnote{To account for a negated between state-constraint $\neg (n_1,l,n_2)$,
we check in every state between $n_1$ and $n_2$ whether 
$\neg l$ holds. If so, we remove the constraint from the
formula. If $\neg (n_1,l,n_2)$ exists when the first action 
of $n_2$ is executed, $\neg l$ is then relevant for it.}

\begin{definition}[Relevant Constraint]\label{def:rel}
Let $d=\tnet{S,\phi}$ be a task network with an 
action $(n:t) \in S$, and $c \in \phi$ 
a between state-constraint or a possibly negated 
before or after state-constraint. Let $c_2$ be the 
non-negated constraint corresponding to $c$. Then, $c$ is 
\defterm{relevant} for executing $n$ relative 
to $d$ if for some literal $l$: 
\begin{itemize}

\item
$c_2 \in \{ (l,n), (l,\first[\{n,\ldots\}]) \}$; or \\\\ for some $x'$ and $n' \not\in \lab(S)$,

\item
$c_2 \in \{ (n',l), 
            (\last[\emptyset],l), 
            (n',l,x'), 
            (\last[\emptyset],l,x') \}$.

\end{itemize}
\end{definition}

The set of  relevant constraints for executing
$n$ relative to $d$ is denoted by $\rel(n,d)$.
For example, if $d$ is the resulting task network after
 the two reductions in our running example,
the relevant constraints for $(11:\heat)$ in 
Figure \ref{fig:ex} is the set:
$\{
\big(\lander(L),\first[11,12]\big),
\big(\didExp (loc1), \first[11,12]\big),
(\neg \stormy, 11),\\ (\cali, 11)
\}$, where the first two constraints are 
`inherited' from $(6:\navigate(L))$.
In the above definition, $n'$ and $\last[\emptyset]$
represent an action that was already executed, whose 
associated after or between state-constraints have
been `propagated' to $n$.

We next define what it means to  `extract' 
the literals from a given set of state constraints.
Let us denote the subset of negated constraints 
as $\rel^-(n,d) = \{c \in \rel(n,d) \mid c$ is a
negated constraint$\}$, and the subset of positive ones as 
$\rel^+(n,d) = \rel(n,d) \setminus \rel^-(n,d)$.
Then, the set of \emph{extracted literals} is denoted
$\rel_l(n,d) = \{l \mid$ literal $l$ occurs in 
$c, c \in \rel^+(n,d)\} \cup \{\neg l \mid$ literal 
$l$ occurs in $c, c \in \rel^-(n,d)\}$.
We can now define what it means for an 
action to be applicable.

\begin{definition}[Applicability]\label{def:app}
Let $\DD=\tuple{\Op,\Me}$ be a domain, $\II$ a 
state, and $d=\tnet{S,\phi}$ a task network with an 
action $(n:t) \in S$ such that $n \in \fst(d)$.
Let $\Phi(n,d,\Op)$ denote  
the precondition and extracted literals, i.e.,
the formula $\pre(t,\Op) 
\land \bigwedge l \in \rel_l(n,d)$.
Then, $n$ is 
\defterm{applicable} in $\II$ relative to $d$ 
and $\Op$ if $\II \models \Phi(n,d,\Op)$. 
\end{definition}

Executing an applicable  action results in changes 
to both the current state and the current task network:
the action is removed from the network's set of tasks, 
and the action's `realised'
constraints, e.g. the relevant ones that do not need to be re-evaluated before executing 
other actions, are removed from the network's constraint formula. The constraints that do need to be re-evaluated are the between state-constraints that 
require literals to hold from the end of an action that 
was executed earlier, up to an action that is yet to be executed.
Formally, given a task network $d=\tnet{S,\phi}$ and  
an action $(n:t) \in S$, we denote by $C_1$ the 
\emph{realised ordering constraints} upon executing 
$n$ (relative to $d$), i.e., the set
\[
\begin{array}{rcl}
\hspace{-2mm}\big\{(x \prec x') \in \phi & \mid & \text{for some}\ x'\ \text{and}\ x \in \{n,\last[\{n\}]\}\big\}\ \cup\\
\hspace{-2mm}\big\{\neg (x' \prec x) \in \phi & \mid & \text{for some}\ x'\ \text{and}\ x \in \{n, \first[\{n,\ldots\}]\}\big\},
\end{array}
\]
where $x'$ represents an action(s) that is
yet to be executed. Notice that a negated 
ordering constraint is realised only if 
one or more (or all) of the actions
corresponding to $x'$ are executed after 
the first (or only) one corresponding to $x$.
Next, we denote by $C_2$ the \emph{realised state
constraints} upon executing 
$n$, i.e., the set obtained from $\rel(n,d)$
by removing any between state-constraint $(x,l,x')$ when $x' \neq n$
and $x' \neq \first[\{n,\ldots\}]$.
Then, we can define the set of \defterm{realised}
constraints upon executing $n$ relative to $d$ as $\fin(n,d) = C_1 \cup C_2$, 
and the result of executing an action
as follows.

\begin{definition}[Action Result]\label{def:actres}
Let $\Op$ be a set of operators, $\II$ a state, 
$d$ a task network, $R$ a set of reduction couples,
$(n:t) \in \fst(d)$ an action, and $\theta$ a substitution.
The \defterm{result} 
of executing $n$ from $\II$ relative to $d, \theta,
R$ and $\Op$, denoted $\res(n,\II,d,\theta,R,\Op)$, is 
 the tuple $\tuple{\tnet{S',\phi'},\II',R}\theta$, where 
\begin{itemize}

\item 
$S' = S \setminus \{(n:t)\}$, where $d=\tnet{S,\phi}$; 
\item 
$\II' = \big(\II 
\setminus \del(t\theta,\Op)\big) \cup \add(t\theta,\Op)$; and

\item
$\phi'$ is obtained from $\phi \setminus \fin(n,d)$
by removing all occurrences of $n$ within 
$\last[]$ expressions.\footnote{We also remove from $\phi'$
any (remaining) constraint of the form $(x,l,x')$ such that $n$ occurs 
in $x'$, i.e., a between state-constraint that holds trivially.}
\end{itemize} 
\end{definition}

Notice that the only possible update to $R$ is a 
substitution of one or more variables (we do not remove 
executed actions from reduction 
couples). Finally, we define an 
execution via action of a task network as the execution 
of (a ground instance of) an applicable primary  action in it.

\begin{definition}[Execution via Action]
\label{def:exec-act}
Let $\DD = \tuple{\Op,\Me}$ be a domain,
$\II$ a state, $R$ 
a set of reduction couples, and $d =
\tnet{S,\phi}$ a task network such that
$\II \models \psi$
for some $\theta$ and action $(n:t) \in \fst(d)$,
where $\psi = \Phi(n,d,\Op)\theta$ is a ground formula.
An \defterm{execution} via \defterm{action}
of $d$ from $\II$ relative to $R$ and $\DD$
is the configuration
$\tuple{d',\II',R',\DD}$, where
$\tuple{d',\II',R'} = \res(n,\II,d,\theta,R,\Op)$.
\end{definition}

Continuing with our running example, let
$\tuple{\tnet{S,\phi},\II,R,\DD}$, with
$S=\{11,12,7,B\}$, be the 
configuration resulting from the two reductions from 
before. Then, an execution via action of $d$ from 
$\II$ relative to $R$ and $\DD$ is the configuration 
$\tuple{\tnet{S',\phi'},\II',R',\DD}$, where 
$\II' = \II \cup \{\camMoved\}$;  set $S' =
\{(12:\move(lan1)),7,B\}$;  formula $\phi'$ 
is obtained from $\phi$ by removing all constraints except for $(\last[11,12] \prec 7)$, which is updated to $(\last[12] \prec 7)$;
and $R'$ is obtained from $R$ by applying  substitution $\{L/lan1\}$.

Observe that the applicability of a method
(relative to the current state) is not checked 
at the point that it is chosen to reduce a task, but
%
\emph{immediately} before executing (for the first time) an associated
primary action---which 
may be after performing further reductions and 
unordered actions. On the other hand, BDI agent 
programming languages such as \AgentSpeak\ 
and \CAN\ 
check the applicability of a relevant recipe
at \emph{some point} before (not necessarily 
just before) executing an associated 
primary action. Thus, in cases where the environment 
changes between checking the recipe's applicability 
and executing an associated primary action (for the first time), and makes the recipe 
no longer applicable, the action will still be 
executed (provided, of course, the action itself 
is applicable). Such behaviour is not permitted
by our semantics.

We  now define the final notion of execution: 
\emph{execution via replacement}, i.e., replacing 
the reductions being pursued for a task 
if 
they have become \emph{blocked}.
Intuitively, this happens when none of the primary 
actions in the pursued reductions are applicable, 
and none of the primary non-primitive tasks have 
a relevant method.

Formally, let $\DD=\tuple{\Op,\Me}$ be a domain, $\II$ 
a state, $d$ a task network, and $\tuple{S,D}$ a 
reduction couple with $S \cap \fst(d) \neq \emptyset$. 
Then,  set $S$ is \defterm{blocked} in 
$d$ from $\II$ relative to $\DD$, denoted $\blocked(S,d,\II,\DD)$, 
if for all $(n:t) \in S \cap \fst(d)$, either $t$ 
is an action and $\II \not\models \Phi(n,d,\Op)$, 
or $t$ is a non-primitive task and $\relv(t,\DD) = \emptyset$. 
Recall that $S$ represents the reductions that are being 
pursued for a particular task (and its subtasks). 


When such pursued reductions are blocked, they are 
replaced by an alternative relevant method-body 
 for the task.
In the definition below, we use the $\first[]$ 
and $\last[]$ constructs (if any) `inserted' into the 
constraint formula by the first reduction of the task 
(Definition \ref{def:red}). Recall that these constructs
represent the `inheritance' of the task's
associated constraints by its descendant tasks.

\begin{definition}[Replacement]\label{def:repl}
Let $d = \tnet{S_d, \phi_d}$ be a task network,
$\tuple{S,D}$ a reduction couple, 
and $d_{new} = 
\tnet{S_{new}, \phi_{new}} \in D$. 
%
%
The \defterm{replacement} of (the elements
of) $S$ in $S_d$ with $S_{new}$ relative to 
$d_{new}$ and $d$, denoted 
$\rep(S,d_{new},d)$, is the task network
%
\begin{eqnarray*}
\tnet{(S_d \setminus S') \cup S_{new}, \psi \land \phi_{new}},
\end{eqnarray*}
where $S' = S \cap S_d$, and $\psi$ is obtained from $\phi_d$ by \textit{(i)} 
replacing any occurrence of (all) the task labels in $\lab(S')$
---within a $\first[]$ or a $\last[]$ expression---with 
the labels in $\lab(S_{new})$, and then \textit{(ii)} 
removing any element mentioning a task label in $\lab(S)$. 
\end{definition}

After a replacement, we need to \emph{update} the 
set of reduction couples accordingly, by doing the 
same
replacement in all  relevant reduction couples. 
In the definition below, the set $S'$ and task network 
$d_{new}$ are as above.

\begin{definition}[Update]
Let $R$ be a set of reduction couples with 
$\tuple{S,D} \in R$, let $S' \subseteq S$, 
and $d_{new} \in D$. 
The \defterm{update} of $S'$ in $S$
with $S_{new}$ relative to $d_{new}$ and $R$,
denoted $\upd(S',S,d_{new},R)$, is the set obtained 
from $R'=\big(R \setminus \{\tuple{S,D}\}\big) 
\cup \{\tuple{S,D \setminus \{d_{new}\}}\}$ by 
replacing any couple $\tuple{S'' \supseteq S, D''}$ 
with $\tuple{(S'' \setminus S') \cup S_{new}, D''}$,
and then 
removing any couple that still mentions an element in $S'$.
\end{definition}

Finally, we combine the two definitions above to 
define the configuration that results from an execution 
via replacement. While we provide a general definition,
 for replacing \emph{any} task's blocked 
(pursued) reductions, one might instead want to, as in 
depth-first search, first replace a \emph{least abstract} task's blocked reductions.
That is, one might want to first consider the \emph{smallest 
replaceable} reduction couples. Formally, given a set 
of reduction couples $R$, a couple $\tuple{S,D} \in R$ is 
a \defterm{smallest replaceable} one in $R$, denoted 
$\tuple{S,D} \in \smallest(R)$, if $D \neq \emptyset$ and 
for each couple $\tuple{S',D'} \in R$, either 
\textit{(a)} $S' \supseteq S$;
\textit{(b)} $S' \subset S$ and $D' = \emptyset$; or
\textit{(c)} $S \cap S' = \emptyset$.

\begin{definition}[Execution via Replacement]\label{def:exec-rep}
Let $\DD$ be a domain, $\II$ a state, 
$d=\tnet{S_d,\phi_d}$ a task network, 
and $R$ a set of reduction couples with an  
$r=\tuple{S,\{d_{new},\ldots\}} \in R$ such that 
$\blocked(S,d,\II,\DD)$ holds.
An \defterm{execution} via \defterm{replacement}
of $d$ from $\II$ relative to $R$ and $\DD$ is the configuration
\begin{eqnarray*}
\tuple{\rep(S,d_{new},d),\II,\upd(S \cap S_d, S, d_{new},R),\DD}; 
\end{eqnarray*}
the replacement is \defterm{complete} if $S \subseteq S_d$
and \defterm{partial} otherwise, and a \defterm{jump} if
$r \not\in \smallest(R)$.
\end{definition}

A complete-replacement represents
the BDI-style searching of an achievement-goal's (i.e., a task's) set 
of relevant recipes in order to find one that is applicable, and a partial-replacement represents \BDI-style recovery from the failure 
to execute (or successfully execute) an action, e.g. due to an environmental change.
We illustrate these  notions  
of replacement with the following examples.

Continuing with our running example, let
$\tuple{\tnet{S,\phi},\II,R,\DD}$ be the 
configuration resulting from the two reductions
from before. Suppose however that the rover's 
instruments were not calibrated, i.e.,
$\II \not\models \cali$. Then, action $(11:\heat)$
is not applicable, and an execution via 
complete-replacement is performed on tasks 
in $\tnet{S,\phi}$ to obtain configuration 
$\tuple{\tnet{S',\phi'},\II,R',\DD}$, 
where $S' = S'' \cup \{7,B\}$;  set
$S'' = \{8, 9, 10\}$; formula 
$\phi'$ is the conjunction of $\phi_4$, and 
$\phi$ updated by, e.g. removing the 
constraints that were copied from $\phi_5$ and replacing constraint
$(\lander(L),\first[11,12])$ 
with  $(\lander(L),\first[8,9,10])$; 
and the set of couples
$R'= 
\{
\tuple{S',\emptyset},
\tuple{S'' \cup \{7\},\{d_1\}},
\tuple{S'',\emptyset}
\}$.
 
Suppose  we now perform two executions via action 
to obtain configuration $\tuple{\tnet{S''',\phi''},\II',R'',\DD}$, 
with $S''' = \{(10:\move(lan1)),7,B\}$ and formula $\phi''$ 
(resp. set $R''$) being $\phi'$ (resp. $R'$) updated to 
account for the executions. Finally, suppose that the battery 
level drops due to the execution of  top-level 
image processing action $(B:\monitor)$, which
makes $\move(lan1)$ no longer applicable. (We will show 
later how $\monitor$ could instead be absent in the initial task network
and arrive `dynamically' from the environment.)
Then, an execution via partial-replacement  will be performed on tasks in $\tnet{S''' \setminus \{(B:\monitor)\},\phi''}$
to obtain  configuration 
$\tuple{\tnet{\{1,2,3\},\phi'''},\II'',R''',\DD}$, 
where $\phi'''$ (resp. $\II''$) is the updated $\phi''$ (resp. $\II'$), and the set
 $R'''= 
\{
\tuple{\{8,9,1,2,3,B\},\emptyset},
\tuple{\{8,9,1,2,3\},\emptyset}
\}$.

\section{Properties of the Formalism}

In this section, we discuss the properties of our formalism,
and in particular how it relates to 
HTN planning.

The properties 
are based on the definition
of an \emph{execution trace}, which formalises the consecutive
execution of a configuration---via reduction, replacement, or action---as in our running example. In what follows, we use   
$\tau \in \exec(d,\II,R,\DD)$ to denote that a configuration $\tau$ is an execution via reduction, action, or replacement of a task network $d$ from a  state $\II$ relative to a set of reduction couples $R$ and a domain $\DD$.

\begin{definition}[Execution Trace]\label{def:trace}
Let $d = \tnet{S_d,\phi_d}$ be a task network, $\II$ a state, and $\DD$ a domain. 
An \defterm{execution trace} $\TT$ of $d$
from $\II$ relative to $\DD$ is any sequence of
configurations $\tau_1 \cdot\ldots\cdot \tau_k$, with
each $\tau_i = \tuple{d_i,\II_i,R_i,\DD}$, such
that $d_1 = d$; 
$\II_1 = \II$; $R_1 = \{\tuple{S_d,\emptyset}\}$;
and $\tau_{i+1} \in \exec(d_i,\II_i,R_i,\DD)$
for all $i \in [1,k-1]$. 
\end{definition}

We also need some auxiliary definitions related 
to execution traces. Consider  configuration 
$\tau_k$ above. First, if $S_k = \emptyset$ 
(where $d_k = \tnet{S_k,\phi_k}$), then the 
trace is \defterm{successful}. Second, if for 
all couples $\tuple{S,D} \in R_k$ we have that 
$S \cap \fst(d_k) \neq \emptyset$ entails both 
$\blocked(S,d_k,\II_k,\DD)$ and $D = \emptyset$, 
then the trace is \defterm{blocked}.
The following theorem states that if a trace is 
successful or blocked as we have  `syntactically' 
defined, then there is no way to `extend' the trace  
further, and vice versa.

\begin{proposition}
Let $\TT$ be an execution trace of a task network 
$d$ from a state $\II$ relative to a domain $\DD$. 
There exists an execution trace $\TT \cdot \tau_1 
\cdot\ldots\cdot \tau_{k}$, with $k>0$, of $d$ (from 
$\II$ relative to $\DD$) if and only if $\TT$ is 
neither successful nor blocked. The inverse also holds.
\end{proposition}
\begin{proof}
If there exists a trace $\TT \cdot \tau_1 
\cdot\ldots\cdot \tau_{k}$ with $k>0$ then $\TT$ cannot 
be successful as its final task network $d_k=\tnet{S_k,\phi_k}$ 
would then not mention any tasks, and thus we cannot `extend'
it to $\tau_1$. The fact that $\TT$ cannot be blocked follows 
from the fact that an execution via replacement, action, or 
reduction of $d_k$ is possible.
Conversely, if $\TT$ is neither successful nor blocked, then
the only reason it would not be possible to `extend' it is if
$S_k \neq \emptyset$ but $\fst(d_k) = \emptyset$. However, 
this is only possible if a method-body exists where its constraint 
formula contains inconsistent (possibly negated) ordering constraints. Such method-bodies are not allowed due to our assumption in Section \ref{sec:prelim}.
The inverse of the theorem is proved similarly.
\end{proof}

The next three properties rely on traces that are 
free from certain kinds of execution. A trace 
$\TT = \tau_1 \cdot\ldots\cdot \tau_k$ is
\defterm{complete-replacement free} if there 
does not exist an index $i \in [1,k-1]$ such that 
$\tau_{i+1}$ is an execution via complete-replacement of $d_i$ from $\II_i$ relative to 
$R_i$ and $\DD$. We define  \emph{partial-replacement free} 
and  \emph{jump free} traces 
similarly.

Given any execution trace, the next theorem states 
that there is an equivalent one---in terms of 
actions performed---that is complete-replacement free. 
Intuitively, this is because, either with some `lookahead' 
mechanism or `luck', a complete-replacement can be avoided 
by choosing a different (or `correct') relevant method-body for a task. 
We define the \emph{actions performed}  by a trace
$\TT=\tau_1 \cdot\ldots\cdot \tau_k$ (or the pursued `solution'), denoted 
$\act(\TT)$, as follows.
Given an index $i \in [1,k-1]$, we first define
$\act(i) = t$ if $S_{i} \setminus S_{i+1} = \{(n:t)\}$ 
and $\tau_{i+1}$ is an execution via action of $d_{i}$ 
from $\II_{i}$ relative to $R_{i}$ and $\DD$; 
otherwise, we define $\act(i) = \epsilon$. Then,
$\act(\TT)$ is $\act(1) \cdot\ldots\cdot \act(k-1)$
with substitution $\theta$ of configuration
$\tau_k$ applied to the sequence.

\begin{theorem}\label{thm:1}
%
Let $\TT$ be an execution trace of a task network
$d$ from a state $\II$ relative to a domain $\DD$.
There exists a complete-replacement free execution 
trace $\TT'$ of $d$ from $\II$
relative to $\DD$ such that $\act(\TT) =\act(\TT')$ 
and $|\TT'| \leq |\TT|$.
\end{theorem}

\begin{proof}
Without loss of generality, we use a slightly modified version of Definition
\ref{def:exec-red} that stores also the 
unique task label that was reduced, i.e., we add tuple 
$\tuple{n,S_n,D \setminus \{d_n\}}$
to $R'$ instead of the one that is currently added
in the definition.
Then, given a tuple $\tau_i=\tuple{d_i,\II_i,R_i,\DD}$ 
occurring in the above execution trace $\TT = \tau_1 \cdot\ldots\cdot \tau_k$, with a tuple
$\tuple{n,S,D} \in R_i$, we say that the set $S'$ is 
an \emph{evolution} of $n$ (relative to the trace
and $i$),
denoted $S' \in \refn(n,\tau_1 \cdot\ldots\cdot \tau_k,i)$,
if $\tuple{n,S',D'} \in R_j$ for some $D'$ and 
$i \leq j \leq k$. 

Consider the smallest $0 < m < k$ such that $\tau_{m+1}$ 
(with each $\tau_i=\tuple{d_i,\II_i,R_i,\DD}$) is 
an execution via complete replacement of task network 
$d_m$ from $\II_m$ relative to $R_m$ and $\DD$.
If there is no such $\tau_{m+1}$ then the theorem
holds trivially; otherwise, $d_{m+1} = \rep(\bar{S},\bar{d}',d_m)$ 
for some $\tuple{\bar{n},\bar{S},D} \in R_m$ and $\bar{d}' \in D$.

Consider prefix $\tau_1 \cdot\ldots\cdot \tau_i$ with 
the smallest $i<m$ such that the `incorrect' reduction 
was performed at $i$, i.e., where $d_{i+1} = \tnet{S_{i+1},\phi_{i+1}}$ 
is an execution via reduction of $d_i$ from $\II_i$ relative to 
$R_i$ and $\DD$, and $\tuple{\bar{n},S^R_{i+1},D} \in R_{i+1}$ 
(for some $S^R_{i+1}$ and $D$) but 
$\tuple{\bar{n},S^R_{i},D} \not\in R_{i}$ (for any $S^R_{i}$), 
where $S^R_{i+1}$ is a set of `ancestors' of $\bar{S}$, 
i.e., $\bar{S} \in \refn(\bar{n}, \TT,i+1)$.

Let $d_{i+1} = \redd(d_i,\bar{n},\bar{d})$ for some 
$(\bar{n}:t) \in S_i$ and $\bar{d} \in \relv(t,\DD)$,
i.e., $\bar{d}$ is the `incorrect' reduction.
Suppose instead that the `correct' one 
was performed on $d_i$, i.e., let tuple 
$\tau_{i+1}' = \tuple{d_{i+1}',\II_{i+1}',R_{i+1}',\DD}$ 
be an execution via reduction
of $d_i$ from $\II_i$ relative to  $R_i$ and $\DD$ 
such that $d_{i+1}' =\redd(d_i,\bar{n},\bar{d}')$ where
 $\bar{d}'$ is from earlier. We now show that all  
executions performed from $\tau_{i+1}$ up to $\tau_m$ (which 
do not involve complete-replacements) can also be 
performed from $\tau_{i+1}'$.

Suppose that there is at least one such execution,
i.e., $i+1<m$. Then $\tau_{i+2}$ is an execution 
via reduction, partial-replacement or action of 
$d_{i+1}$ from $\II_{i+1}$ relative to $R_{i+1}$ and 
$\DD$. Let $\lab(S_{bef}) \subset \lab(S_{i+1})$ be the task 
that was executed or reduced, or the tasks that were 
replaced, i.e., the largest set such that
$\lab(S_{bef}) \cap \lab(S_{i+2}) = \emptyset$; in the case of 
an execution via reduction or partial-replacement, 
let $\lab(S_{aft}) \subset \lab(S_{i+2})$ be the new tasks, 
i.e., the largest set such that 
$\lab(S_{aft}) \cap \lab(S_{i+1}) = \emptyset$.
If $\lab(S_{bef}) \subset \lab(S_{i+1}')$, i.e., the 
execution is `relevant' to $d_{i+1}'$ and the execution
is not a reduction of some `descendant' of 
$S^R_{i+1}$,\footnote{The execution 
cannot be via complete- or partial-replacement 
of a descendant of $S^R_{i+1}$ either, as the first 
execution via complete-replacement of $\bar{S}$ 
happens at index $m$.}
we then show that there exists also a corresponding 
tuple $\tau_{i+2}'$
that 
is an execution via reduction, partial-replacement 
or action of $d_{i+1}'$ from $\II_{i+1}'$ relative to 
$R_{i+1}'$ and $\DD$, such that 
$\lab(S_{bef}) \cap \lab(S_{i+2}') = \emptyset$, and in the 
case of an execution via reduction or partial-replacement, $\lab(S_{aft}) \subset \lab(S_{i+2}')$ and 
$\lab(S_{aft}) \cap \lab(S_{i+1}') = \emptyset$.


There are two main cases to consider: an execution
via action and partial-replacement. 

In the case 
of an execution via partial-replacement, observe 
from Definition \ref{def:exec-rep} that
all tasks in $S_{bef}$ are blocked in $d_{i+1}$ 
from $\II_{i+1}$ relative to $\DD$. The same applies 
to $S_{bef}$ in $d_{i+1}'$ from $\II_{i+1}'$ relative 
to $\DD$ for the following two reasons. Consider 
any primitive task $(n':t') \in S_{bef}$ (which 
is not applicable in $\II_{i+1}$ relative to $d_{i+1}$).
First, observe from Definition \ref{def:red}
that $\phi_{i+1}$ and $\phi_{i+1}'$ are 
identical except for the tasks and constraints
that were introduced by the two different 
reductions of $d_i$ above. Second, observe
from Definition \ref{def:rel} that any constraint 
occurring in $\phi_{i+1}$ and $\phi_{i+1}'$ 
containing expression $\first[n',\ldots]$ is relevant 
to $n'$ irrespective of the other task labels
that occur in the expression. Similarly, any 
constraint occurring in $\phi_{i+1}$ and $\phi_{i+1}'$ 
containing expression $\last[n',\ldots]$ 
is not relevant to $n'$ irrespective of the 
other task labels that occur in the expression. 
The same applies when $n'$ does not occur 
within such expressions.

In the case of an execution via action, i.e.,
$S_{bef}=\{(n':t')\}$ for some $(n':t')$, 
$n'$ is applicable in $\II_{i+1}'=\II_{i+1}$ 
relative to both $d_{i+1}$ 
and $d_{i+1}'$ for the same reasons as above.

It is not difficult to see that the remaining
relevant executions performed from $d_{i+2}$ up
to $d_m$ (if any)---which do not involve complete-replacements---can also be performed from $d_{i+2}'$.
Let $\TT'=\tau_1 \cdot\ldots\cdot \tau_i \cdot 
\tau_{i+1}' \cdot\ldots\cdot \tau_j'$ be the trace 
corresponding to those remaining executions from 
$d_{i+1}'$, appended to the prefix that ends just
before the `incorrect' reduction.
Since the set $\bar{S}$ above represents the descendants of 
$\bar{n}$ (with all actions left in-tact), it follows 
from Definitions \ref{def:red} and \ref{def:repl} 
that $\tau_{j}' = \tau_{m+1}$.

If $m+1 = k$, then $\TT'$ is a complete-replacement
free execution trace of $d$ from $\II$ relative to 
$\DD$, and the theorem holds. Otherwise, we first
create an `adjusted' copy of the suffix from index $m+2$ 
of the original trace by inserting the `incorrect' 
task network $\bar{d}$ (which $\TT'$ did not use). 
Let the trace $\hat{\tau}_1 \cdot\ldots\cdot \hat{\tau}_x = 
\tau_{m+2} \cdot\ldots\cdot \tau_k$. Then,
we replace each tuple $\tuple{\bar{n},S,D}$
occurring in $\hat{\tau}_1 \cdot\ldots\cdot \hat{\tau}_x$
with tuple $\tuple{\bar{n},S,D \cup \{\bar{d}\}}$ to
obtain the new trace
$\TT''= \TT' \cdot \hat{\tau}_1 \cdot\ldots\cdot \hat{\tau}_x$.

Finally, we can now remove the first execution via 
complete-replacement from $\TT''$ as we did before 
with trace $\TT$, and 
then continue this process (a finite number of 
times) with the resulting traces, until one is 
obtained where there is no execution via 
complete-replacement.
\end{proof}

An equivalent complete-replacement free
trace may, however, unavoidably specify one
or more replacements that are jumps---where 
the smallest replaceable reduction couples 
were skipped. To see why this holds, 
consider once again our running example,
but suppose that the constraints associated
with $\neg \stormy$ do not exist in $\phi_4$
and $\phi_5$ in Figure \ref{fig:ex}.%
\footnote{One could imagine a setting where
\textit{(i)} $\neg \stormy$ should only be checked immediately before action $\move$, 
and \textit{(ii)} it is 
not undesirable to do actions $\calibrate$ (if $\neg \cali$ holds) and $\heat$, even if action $\move$ turns out to be non-applicable.}
Suppose also that after the first reduction 
(of task $A$),  task $6$ is reduced using 
method $m_4$ instead of $m_5$, which  means
that the complete-replacement in the previous 
example will not occur. The resulting 
set of reduction couples will then 
contain the couple $\tuple{S,\{d_5\}}$, with 
$S=\{8,9,10\}$, instead of the couple 
$\tuple{S,\emptyset}$ in the previous example
(after the complete-replacement was performed). 
Thus, after the two executions via action of 
tasks $8$ and $9$ as before, the subsequent 
partial-replacement must `skip' couple 
$\tuple{S,\{d_5\}}$, which is the smallest 
replaceable one, and `jump' to  couple 
$\tuple{ S \cup \{7\}, \{d_1\} }$ in order
to avoid performing $(11:\heat)$. Intuitively, 
the jump is needed to `mimic' 
the actions yielded by the trace 
depicted by the previous example, which 
considered $d_5$ but then removed it (via 
the complete-replacement) because it was
not applicable. This observation is stated
formally below.

\begin{proposition}
There exists a domain $\DD$, state $\II$, task
network $d$, and an execution trace $\TT$ of $d$ 
from $\II$ relative to $\DD$ such that any
complete-replacement free execution trace $\TT'$ of $d$ from 
$\II$ relative to $\DD$ is not jump free when
$\act(\TT) =\act(\TT')$.
\end{proposition}
\begin{proof}
This follows from the example above.
\end{proof}

The next result makes the link concrete between our
\HTN\ acting formalism and \HTN\ planning. It states 
that the solution yielded
by any execution trace that is successful and free
from partial-replacements can also be yielded via 
HTN planning. Conversely, given any  
HTN planning solution, there exists such an execution 
trace that yields it. The trace must be free from partial-replacements because such behaviour is specific to BDI-style recovery from runtime failure.

\begin{theorem}
Let $\DD$ be a domain, $\II$ a state, and $d$ a 
task network. Then, $\sigma \in \sol(d,\II,\DD)$ 
if and only if there exists a partial-replacement 
free and successful execution trace $\TT$ of $d$ 
from $\II$ relative to $\DD$ such that 
$\sigma = \act(\TT)$.
\end{theorem}
\begin{proof} 
This proof relies on some auxiliary
functions. 
First, given a task label $n$ that appears
in a sequence of task networks $\vec{d}$, 
we define $f(n,\vec{d})$ 
(denoted $f(n)$ when $\vec{d}$ is obvious from 
context) as the index $i>1$ in $\vec{d}$ such 
that $n \in \lab(S_i)$ but
$n \not\in \lab(S_{i-1})$, where  
$\tnet{S_j,\phi_j}$ denotes the element at index
$j$ in $\vec{d}$; if there is no such index
$i$, we take $i=1$.

Second, given a task label $n$ that appears
in a trace $\TT$, we define 
$g(n,\TT)$ (denoted $g(n)$ when $\TT$ is 
obvious from context) as the index $i>0$ in 
$\TT$ such that $d_{i+1}$ is an execution
via action of $d_i$ from $\II_i$ relative to 
$R_i$ and $\DD$, and $n \in \lab(S_i)$ but 
$n \not\in \lab(S_{i+1})$, where each 
$\tau_j \in \TT$ is of the form 
$\tuple{d_j = \tnet{S_j,\phi_j},\II_j,R_j,\DD}$.

Finally, we sometimes assume that functions 
$\sol$ and $\act$ do not remove task labels, 
i.e., the latter, or an element of the former, 
can be of the form $(n_1:t_1) \cdot\ldots\cdot (n_m:t_m)$.

We shall now prove each direction of the theorem.\\

\noindent{\bf (\noindent\textbf{$\Longrightarrow$})}
Let us assume that $\sigma \in \sol(d,\II,\DD)$.
Then, from the definition of an \HTN\ solution, there 
exists a sequence $\vec{d}=d_1 \cdot\ldots\cdot d_m$, 
with $d=d_1$, such that
\textit{(i)}
for all $i \in [1,m-1]$,
$d_{i+1} = 
\redd(d_i=\tnet{S_i,\phi_i},n_i,\hat{d}_i)$
for some $\hat{d}_i \in \relv(t,\DD)$ and 
$(n:t) \in S_i$;  
\textit{(ii)}
$\sigma \in \comp(d_m,\II,\DD)$; 
\textit{(iii)}
$d_1 \cdot\ldots\cdot d_{f(n_1)}$, with 
$(n_1:t_1) \in \sigma$ (each $(n_j:t_j)$
denotes the element at index $j$ in $\sigma)$, 
is the shortest possible sequence of reductions
that yields $n_1$; and similarly, 
\textit{(iv)}
for any pair 
$(n_i:t_i), (n_{i+1}:t_{i+1}) \in \sigma$, the 
sequence $d_{x} \cdot\ldots\cdot d_{y}$ is the
shortest possible one such that $x=f(n_i)$ 
and $y=f(n_{i+1})$, unless $f(n_{i+1}) > f(n_i)$
(i.e., $n_{i+1}$ is yielded in the process
of yielding $n_{i}$), in which case $y=f(n_i)$. 
Since the order of reductions
does not matter \cite{ErolHN:TR94},
we can always obtain such a sequence $\vec{d}$
by `re-ordering' the reductions in a given sequence.

We now show that $\vec{d}$ also has a 
corresponding trace $\TT$ as above. To 
this end, we first prove the following
weaker theorem: there exists a partial-replacement 
free execution trace $\TT$ of $d$ from
$\II$ relative to $\DD$ such that $\sigma 
= \act(\TT)$.

We prove this by induction on the
length of the prefixes of $\sigma$.
For the base case, we consider
only the first action 
$(\bar{n}:\bar{t}) \in \sigma$.
Let $\tau_1 \cdot\ldots\cdot \tau_{k}$,
with each $\tau_i = 
\tuple{d_i=\tnet{S_i,\phi_i},\II,R_i,\DD}$,
be the trace of $d_1$
from $\II$ relative to $R_1$ and $\DD$ 
such that
\textit{(i)}
$k = f(\bar{n})$; 
\textit{(ii)}
each 
$d_{i+1} = \redd(d_i,n_i,\hat{d}_i)$
for all $i \in [1,k-1]$, where each $n_i$ 
and $\hat{d}_i$ are as above for 
$\vec{d}$; and 
\textit{(iii)}
$R_1=\{\tuple{S_1,\emptyset}\}$.
Since $(\bar{n}:\bar{t})$ is the first task 
in $\sigma$ (which is executable in $\II$), 
the precondition of $\bar{t}$ holds in $\II$,
and so do any `relevant' constraints in $\phi_m$ of the form 
$(l,\bar{n})$, $(l,\first[\bar{n},\ldots])$, 
or their negations. From this it follows that 
$\II \models \Phi(\phi_k,\bar{n},\Op)\theta$ also 
holds (Definition \ref{def:app}):%
\footnote{Substitution $\theta$ must be a subset 
of the one used to compute $\comp(d_m,\II,\DD)$.} 
any such constraints will also occur in $\phi_k$ 
(though a constraint in $\phi_k$ mentioning 
a $\first[N]$ expression with $\bar{n} \in N$
might have fewer elements in $N$ than 
the constraint's `evolution' in $\phi_m$)
and no more before state-constraints can occur in
$\phi_k$ that are relevant to $\bar{n}$.
The same applies for ordering constraints 
associated with $\bar{n}$.
%
%
Then, we can take trace $\TT = 
\tau_1 \cdot\ldots\cdot \tau_{k} \cdot 
\tuple{\tnet{S',\phi'},\II',R',\DD}$,
where the last tuple in the trace is 
an execution via action of $d_k$ from 
$\II$ (relative to $R_{k}$ and $\DD$)
such that $\bar{n} \in \lab(S_k)$ 
but $\bar{n} \not\in \lab(S')$.
Thus, the weaker theorem above holds 
in the base case.

For the induction hypothesis, we
assume that the weaker theorem holds 
for any prefix of $\sigma$ of length 
up to $\ell < |\sigma|$.

We now show that the weaker theorem also 
holds for
the prefix of $\sigma$ of length $\ell+1$. 
Let $(\bar{n}^{\ell}:\bar{t}^{\ell})$ and 
$(\bar{n}^{\ell+1}:\bar{t}^{\ell+1})$ be 
the actions at indices $\ell$ and 
$\ell+1$ in $\sigma$, respectively. Let
$j = f(\bar{n}^{\ell})$, and $k = f(\bar{n}^{\ell+1})$
if $f(\bar{n}^{\ell+1}) > f(\bar{n}^{\ell})$,
and $k=j$ otherwise.

From the induction hypothesis, we know 
that there exists a trace $\TT^{\ell}$ of $d$
from $\II$ relative to $R_1 = \{\tuple{S_1,\emptyset}\}$
and $\DD$ for the prefix of $\sigma$ of length $\ell$.
Let $\tuple{d^{\ell},\II^{\ell},R^{\ell},\DD}$
be the last configuration in $\TT^{\ell}$, and let 
$x = |\TT^{\ell}|-\ell$, i.e., $x$ is the
index in $\vec{d}$ where the next reduction%
---immediately after yielding $\bar{n}^{\ell}$---%
is performed.
Moroever, if $k=j$, let
$\TT^{\ell+1} = \TT^{\ell}$; otherwise (if 
$k > j$), let $\TT^{\ell+1} = 
\TT^{\ell} \cdot \tau_1' \cdot\ldots\cdot 
\tau_{k-j}'$, where for each $i \in [1,k-j-1]$:
\textit{(i)} $\tau_i'$ is of the form
$\tuple{d_i'=\tnet{S_i',\phi_i'},\II^{\ell},R'_i,\DD}$;
\textit{(ii)} 
$\tau_{i+1}'$ is an execution via reduction 
of $d'_i$ from $\II^{\ell}$ relative to $R'_i$ 
and $\DD$, and in particular, $d'_{i+1} = 
\redd(d'_i,n_{x+i},\hat{d}_{x+i})$;%
\footnote{Substitution
$\theta'$ must be applied to $\hat{d}_{x+i}$ 
to account for executions via action, where
given $(n_{x+i}:t_{x+i})$ from before,
$t' = t_{x+i}\theta'$ and 
$(n_{x+i}:t') \in S'_i$.
}
and
\textit{(iii)} 
$\tau_1'$ is an execution via reduction 
of $d^{\ell}$ from $\II^{\ell}$ relative
to $R^{\ell}$ and $\DD$, and in particular, 
$d'_1 = \redd(d^{\ell},n_x,\hat{d}_x)$.
(Each $n_j$ and $\hat{d}_j$ are as before for 
$\vec{d}$.)
 
Finally, we claim that we can take the trace
$\TT = \TT^{\ell+1} \cdot \tuple{d',\II',R',\DD}$, 
where 
$\tuple{d^{\ell+1},\II^{\ell+1},R^{\ell+1},\DD}$
is the last configuration in $\TT^{\ell+1}$
and $d'$ is an execution via action of 
$d^{\ell+1}$ from $\II^{\ell+1}$ relative 
to $R^{\ell+1}$ and $\DD$ such that 
$\bar{n}^{\ell+1}  
\in \lab(S^{\ell+1})$ but 
$\bar{n}^{\ell+1} 
\not\in \lab(S')$. This claim holds true
for the following three reasons.
First, since $\bar{n}^{\ell+1}$ 
is executable after $\bar{n}^{\ell}$ 
in $\sigma$, the preconditions and before constraints 
associated with $\bar{n}^{\ell+1}$ in $d^{\ell+1}$ 
also hold in $\II^{\ell+1}$, similarly to the base 
case above.
The second requirement for $\bar{n}^{\ell+1}$ to
be applicable in $\II^{\ell+1}$ is for any after 
state-constraint of the form $(\bar{n}^{\ell},l)$ occurring
in $d^{\ell+1}$ to hold in $\II^{\ell+1}$ (Definition 
\ref{def:rel}). This follows from the fact that
the same constraint in $\phi_m$ (up to substitution) 
holds in $\II^{\ell+1}$.
A similar argument applies to the 
case where constraints of the form 
$(\last[\bar{n}^{\ell},\ldots],l)$ 
and $(x,l,x')$ occur in $d^{\ell+1}$.
Third, by Definition \ref{def:actres},
once such a constraint is checked against
a state, it is immediately removed from 
the constraint formula (except possibly 
between state-constraints, which are 
removed once they are similarly
`satisfied'). 

Thus, $\TT$ is a partial-replacement free
execution trace of $d$ from $\II$ relative
to $\DD$ with $\act(\TT) = \sigma$. Trace
$\TT$ is a successful execution trace
because $d_m$ is a primitive
task network, and $\sigma$ represents exactly
the tasks occurring in $d_m$.\\

\noindent{\bf (\noindent\textbf{$\Longleftarrow$})}
Let us assume that there exists a partial-replacement 
free and successful execution trace $\TT$ of $d$ 
from $\II$ relative to $\DD$. We show then that
$\act(\TT) \in \sol(d,\II,\DD)$ also holds, that
is, there exists a sequence $\bar{d}_1 
\cdot\ldots\cdot \bar{d}_j$ such that
$d=\bar{d}_1$, 
$\act(\TT) \in \comp(\bar{d}_j,\II,\DD)$
and $\bar{d}_{i+1} = \redd(\bar{d}_i =
\tnet{\bar{S}_i,\bar{\phi}_i},\bar{n},\bar{d})$ 
for some $(\bar{n}:\bar{t}) \in \bar{S}_i$,
and $\bar{d} \in \relv(\bar{t},\DD)$ for
all $i \in [1,j-1]$.
In other words, we need to show that \cite{ErolHN:AMAI96}:
\textit{(i)} $\act(\TT)$ is a permutation of a ground
instance of $\bar{S}_j$ (with task labels removed), 
and that the following holds for any prefix of 
$\act(\TT)$, where for a given prefix of $\act(\TT)$:


\begin{enumerate}

\item[\textit{(ii)}]
the prefix is executable in $\II$, i.e.,
the first action in the prefix is executable
in $\II$, and any other action in the prefix
is executable in the state resulting from
executing the previous action; and 

\item[\textit{(iii)}]
any constraint in $\bar{\phi}_j$ that is 
relevant to an action or a pair of actions 
in the prefix is satisfied relative to the
prefix and $\II$.
\end{enumerate}

Let $\tau_1 \cdot\ldots\cdot \tau_m$ be the
complete-replacement free extraction of $\TT$
(Theorem \ref{thm:1}).
Observe that for each $i \in [1,m-1]$, either
$\tau_{i+1}$ is an execution via action of 
$d_i$ from $\II_i$ relative to $R_i$ and
$\DD$, or $d_{i+1} = \redd(d_i,n_i,\hat{d}_i)$
for some $\hat{d}_i \in \relv(t_i,\DD)$ 
and $(n_i:t_i) \in S_i$, where $\tau_i =
\tuple{d_i=\tnet{S_i,\phi_i},\II_i,R_i,\DD}$.
We prove parts \textit{(ii)} and \textit{(iii)} 
by induction on the lengths $k$ of the 
prefixes of the trace where $k = g(\bar{n})$ 
for some $\bar{n}$. 

For the base case, consider the smallest 
such prefix $k = g(\bar{n})$ for some 
$\bar{n}$, and
%
%
the sequence of task networks
$d_1 \cdot\ldots\cdot d_k$ corresponding
to the trace.
Then, since $(\bar{n}:\bar{t}) \in \fst(d_k)$ 
for some $\bar{t}$, and by Definition 
\ref{def:app}, the precondition of $\bar{t}$ 
holds in $\II$, and any (possibly negated) 
before state-constraint in $\phi_k$ that 
is associated with $\bar{n}$ is also satisfied 
in $\II$, parts \textit{(ii)} and \textit{(iii)}
above hold for prefix $(\bar{n}:\bar{t})$
of $\act(\TT)$, similarly 
to the inductive case from before.

For the induction hypothesis, let $\TT^{\ell}$ 
(resp. $\TT^{\ell+1}$) be any prefix of $\TT$ of length 
up to $\ell=g(\bar{n}^{\ell})$ (resp. $\ell+1$) for
some $\bar{n}^{\ell}$, with $k \leq \ell < m$.
(The step that yields $\tau_m$ must be an execution
via action, as $\TT$ is successful.)
Let
$\sigma^{\ell}$ be the corresponding subplan, i.e., $\sigma^{\ell}=\act(\TT^{\ell+1})$. Then,
we assume that
parts \textit{(ii)} and \textit{(iii)} above 
hold for prefix $\sigma^{\ell}$.

Let $c>0$ be the smallest number such that
$\ell+c=g(\bar{n}^{\ell+c})$ for some  
$\bar{n}^{\ell+c}$. We now show that parts
\textit{(ii)} and \textit{(iii)} also hold
for subplan $\sigma^{\ell+c}=\act(\TT^{\ell+c+1})$, 
where $\TT^{\ell+c+1}$ is the prefix of
$\TT$ of length $\ell+c+1$.

Let $\vec{d^{\ell}}$ be the sequence of task networks
corresponding to $\TT^{\ell}$ and let $d^{\ell}$ be 
the last task network in $\vec{d^{\ell}}$.
If $c=1$, then let $\vec{d^{\ell+c}}=\vec{d^{\ell}}$.
Otherwise, consider the sequence of `reductions'
$(n_{\ell+1},\hat{d}_{\ell+1}) \cdot\ldots\cdot 
(n_{\ell+c-1},\hat{d}_{\ell+c-1})$ performed in the trace 
$\tau_1 \cdot\ldots\cdot \tau_m$ above.
Let $\vec{d^{\ell+c}} = \vec{d^{\ell}} \cdot
d_1' \cdot\ldots\cdot d_{c-1}'$ where $d_1' = 
\redd(d^{\ell},n_{\ell+1},\hat{d}_{\ell+1})$
and $d_{i+1}' = 
\redd(d_i',n_{\ell+1+i},\hat{d}_{\ell+1+i})$ 
for all $i \in [1,c-2]$.%
\footnote{We must also apply to $\vec{d^{\ell+c}}$
the substitutions performed so far, i.e., element
$\theta$ in configuration $\tau_{\ell+c+1}$.} 
That is, we append a new sequence of task
networks to the one corresponding to the induction
hypothesis. Observe that parts \textit{(ii)} and 
\textit{(iii)} above hold for subplan
$\sigma^{\ell+c}=\sigma^{\ell} \cdot 
(\bar{n}^{\ell+c}:\bar{t}^{\ell+c})$
for the following main reasons.
First, any (possibly negated) constraint in 
the constraint formula in $d^{\ell}$ that is 
relevant to $\bar{n}^{\ell}$ and satisfied 
relative to 
$\sigma^{\ell}$ and $\II$ is still a constraint
occurring in $d_{c-1}'$ (though a possibly 
`evolved' one---e.g. with variables replaced 
by constants or more elements added to the
set of task labels in expressions of the form 
$\first[\bar{n}^{\ell},\ldots]$) that is relevant 
to $\bar{n}^{\ell}$ and satisfied relative to 
$\sigma^{\ell+c}$ and $\II$. 
Second, $\bar{t}^{\ell+c}$ is executable (in 
the state resulting from applying $\sigma^{\ell}$
to $\II$) by Definition \ref{def:exec-act}, and
any before state-constraints relevant to 
$\bar{n}^{\ell+c}$, and after state-constraints 
relevant to $\bar{n}^{\ell}$ are satisfied for 
the reasons discussed in the previous inductive case.
The case of 
between state-constraints is proved similarly.

Thus, there exists a sequence of task networks
$\bar{d}_1 \cdot\ldots\cdot \bar{d}_j$ as above,
such that \textit{(ii)} and \textit{(iii)}
hold for $\act(\TT)$. Finally, point 
\textit{(i)} above also holds for $\act(\TT)$ due to 
$\TT$ being successful.
\end{proof}

If a trace is not free 
from partial-replacements, it may not be
possible to obtain its solution via  HTN 
planning (given the same inputs). A similar 
property exists in the \CANPLAN\ semantics: BDI-style 
recovery from failure enables solutions that 
cannot be found using \CANPLAN's built-in 
HTN planning construct.

\begin{theorem}
There exists a domain $\DD$, state $\II$, task 
network $d$, and successful execution trace 
$\TT$ of $d$ from $\II$ relative to $\DD$
such that $\act(\TT) \not\in \sol(d,\II,\DD)$.
\end{theorem}
\begin{proof}
Consider the trace from our running
example, up to the point where an execution via partial-replacement is performed using method $m_1$. If the resulting
task network is successfully executed, we
get the solution corresponding to the sequence 
of action labels $8 \cdot 9 \cdot B \cdot 
1 \cdot 4 \cdot 5 \cdot 3$, which is not an \HTN\
solution; for example, an \HTN\ solution cannot contain
(the actions corresponding to) both $8$ and $1$.
\end{proof}

\section{An Algorithm for HTN Acting}
\label{sec:algo}

In this section we present the $\Exec$ algorithm 
for \HTN\ acting, which combines our formalism
with the processing of exogenous events.
In the algorithm we use $S_{nop}$ 
to denote the initial set of tasks $\{(0:\nop)\}$,
and $\topp(R)$ to denote the (unique) set $S$ of tasks in the 
`top level' reduction couple, given a set of reduction
couples $R$, i.e., the 
couple $\tuple{S \supseteq S_{nop},\emptyset}$. %
The algorithm takes the current state and  HTN 
domain as input and continuously performs  
 two main steps as follows.

\begin{algorithm}[t]
\small
\caption{$\Exec(\II,\DD)$}
\begin{algorithmic}[1]



\STATE $d \leftarrowtwo \tnet{S_d,\phi_d} =
   	    \tnet{S_{nop},\true}$ \COMMENT{Initial task network}


\STATE 
$\textbf{\bar{T}} \leftarrowtwo S_{nop}$;
$R \leftarrowtwo \{\tuple{S_{nop},\emptyset}\}$;
$\TTT \leftarrowtwo \tuple{d,\II,R,\DD}$

\WHILE{\true} \label{alg:outLoop}

   \STATE Set $\textbf{T}$ to the possibly empty set of newly observed tasks\label{alg:obs}


   \STATE $\textbf{T}' \leftarrowtwo \{(n:t) \mid t \in \textbf{T}, n$ is a unique task label$\}$

   \STATE $\textbf{\bar{T}} \leftarrowtwo \textbf{\bar{T}} \cup \textbf{T}'$\label{alg:add} \COMMENT{Store all newly observed tasks}

   \STATE $d' \leftarrowtwo \tnet{S_d \cup \textbf{T}', \phi_d}$\label{alg:add2}


   \STATE $R' \leftarrowtwo \big(R \setminus \{\tuple{\topp(R),\emptyset}\}\big) 
		\cup \{\tuple{\topp(R) \cup \textbf{T}', \emptyset}\}$\label{alg:add3}
		

   \IF{$\textbf{T} \neq \emptyset$}
   
    \STATE $\TTT \leftarrowtwo \TTT \cdot \tuple{d',\II,R',\DD}$\label{alg:append}
   \ENDIF
   
   \IF[Below Def. \ref{def:trace}]{$\TTT$ is neither successful nor blocked}\label{alg:traceIf}  


   	\STATE Set $\tuple{d,\II,R,\DD}$ to an element of  $\exec(d',\II,R',\DD)$\label{alg:exec}  

   	\STATE $\TTT \leftarrowtwo \TTT \cdot \tuple{d,\II,R,\DD}$\label{alg:traceadd} 

   \ENDIF

\ENDWHILE \label{alg:outloopend}


\end{algorithmic}
\end{algorithm}

\textbf{Step 1.}
The algorithm `processes' newly observed 
(external) tasks (if any) and inserts them as top-level tasks 
to a copy of the current configuration's task network 
$d$ and set of reduction couples $R$ (lines 
\ref{alg:obs} to \ref{alg:add3}), which are  
used to create the `next' configuration.

Such tasks could be the initial requests, 
for example to transfer 
the soil data and sample and then recharge, or
requests that arrive later, possibly while other
tasks are  being 
achieved. For example, task $\monitor$ could be a newly
observed task in the iteration following the execution 
of the actions corresponding to task labels $8$ and $9$ in method $m_4$ (as
opposed to $\monitor$ being an initial request).
%
A newly observed task could also represent an exogenous 
event  triggered by a change in the environment; 
for example, the arrival of primitive task $\stormact$ could represent the event 
that it has just become stormy, and it could have  the add-list 
$\{\storm\}$, which will be applied to the agent's 
state 
when the task is executed.
Given a domain $\DD = \tuple{\Op,\Me}$, we stipulate that any 
newly observed task $t$ is such that
$\pre(t,\Op) = \true$ if $t$ is primitive, and 
$\relv(t,\DD) \neq \emptyset$ otherwise.

\textbf{Step 2.}
If one or more new tasks were indeed observed,  the 
corresponding `next' configuration is appended (line 
\ref{alg:append}) to the current `dynamic' execution 
trace, or \emph{d-trace} $\TTT$. A d-trace is slightly
different to an execution trace 
(Definition \ref{def:trace}) in that the former may 
include tasks that are not just obtained by reduction
but also  dynamically from the environment.
%
%
If an execution via reduction, action, or replacement 
is possible from the last configuration in the d-trace
(line \ref{alg:traceIf}), 
the execution is then performed and the resulting configuration 
is appended to the trace (lines \ref{alg:exec} and 
\ref{alg:traceadd}).

The following theorem states that any d-trace produced by the 
algorithm is sound, i.e.,
%
any such d-trace, which may include 
new tasks observed over 
a number of iterations, 
is equivalent to some (standard) execution trace 
such that all of those 
tasks are present in the first configuration, but  their 
execution is `postponed'. 

\begin{theorem}
Let state $\II_{in}$ and domain $\DD$ 
be the inputs of algorithm $\Exec$. 
Let $\textbf{\bar{T}}_{\ref{alg:outloopend}}$ 
and $\TTT_{\ref{alg:outloopend}}$ 
be the values of variables 
$\textbf{\bar{T}}$ and 
$\TTT$, respectively,
on reaching line \ref{alg:outloopend} in the 
algorithm (after one or more iterations).
Then, $\act(\TTT_{\ref{alg:outloopend}}) = 
\act(\TT)$ for some execution trace $\TT$ of 
task network
$\tnet{\textbf{\bar{T}}_{\ref{alg:outloopend}},\true}$ 
from $\II_{in}$ relative to $\DD$. 
\end{theorem}
\begin{proof}
D-trace $\TTT$ in the algorithm, which is incrementally built, is similar to an execution trace, except for \textit{(i)} the initial `empty'
task network of $\TTT$; and \textit{(ii)}  the task networks appended in line
\ref{alg:append} to account for newly observed tasks. We obtain 
an execution trace from $\TTT_{\ref{alg:outloopend}}$ as follows: 
take the last element $\tau_j \in \TTT_{\ref{alg:outloopend}}$ (each $\tau_k = \config{\tnet{S_k,\phi_k},\II_k,R_k,\DD}$) such that
$S_j \subset S_{j+1}$, i.e., there are newly observed tasks
in $S_{j+1}$; remove $\tau_{j+1}$ from $\TTT_{\ref{alg:outloopend}}$; 
 add the elements in $S_{j+1} \setminus S_j$ to each 
$S_i$ and $\topp(R_i)$, for $i \in [1,j]$; and repeat 
these steps on the resulting d-traces until an 
execution trace is obtained.

To see why `propagating' tasks up a d-trace as 
above does not make the latter invalid, consider a 
tuple $\tau_j$ (with $j>0$) in the original $\TTT_{\ref{alg:outloopend}}$ 
such that $S_j \subset S_{j+1}$. Let us now add any 
task $(n:t) \in S_{j+1} \setminus S_j$ to $S_i$ and
$\topp(R_i)$ for some $i<j$.
Since no constraints are added to $\phi_i$ 
and no existing ones in $\phi_i$ 
are modified,
 any other task 
$(n':t') \in S_i$ (with $n' \neq n$) that can (resp. cannot) 
be executed (from $\II_i$ relative to $R_i$ and $\DD$) 
when $(n:t)$ is not in $S_i$, still can (resp. cannot) be 
executed when $(n:t)$ is in $S_i$.
\end{proof}

\section{Discussion and Future Work}

While some implementations of \HTN\ acting frameworks do exist in the literature, this paper has, for the first time, provided a formal framework, by using the most general \HTN\ planning syntax and building on the core of its semantics. In doing so, we have carried over some of the advantages of the \HTN\ planning formalism, such as the ability to flexibly interleave the actions associated with a method \cite{deSilva:2017}, and to check a method's applicability \emph{immediately} before first executing an action. We have also compared \HTN\ acting to \HTN\ planning, and to a \BDI\ agent programming language.

We could now explore adding a `controlled' and `local' account of \HTN\ planning into \HTN\ acting. The result should be a similar semantics to \CANPLAN, which allows a \BDI\ agent to perform \HTN\ planning but only from user-specified points in a hierarchy. 
One approach might be, given a ground non-primitive task $t$, to use the construct $\Plan(t)$ to indicate that \HTN\ planning (as opposed to an arbitrary reduction) must be performed on $t$, and to define the new notion `execution via \HTN\ planning'. Given a current configuration $\tuple{d,\II,R,\DD}$ with task network $d = \tnet{\{(n:\Plan(t)),\ldots\},\phi}$, the definition would, for example, check whether there exists a ground instance $d_t'$ of a  method-body $d_t \in \relv(t,\DD)$ such that $\sol(d_t',\II,\DD) \neq \emptyset$ holds (defined in Section \ref{sec:back-htn}).

We could also investigate an improved semantics where a `tried' method-body is re-tried to achieve a task. Recall that when a relevant method-body is selected to reduce a task, and the body 
turns out to be `non-applicable' (i.e., it is unable
to execute any of its tasks) in the current state, we 
consider the body to have been `tried', in the same way that
we consider a body 
to have been tried if it fails (becomes blocked) 
\emph{during} execution, e.g. due to an environmental 
change. 
To enable re-trying a body that was not applicable 
in an earlier state, we should at the least be able to 
check whether that state is different to the current  
one (both of which are sets of ground atoms). 
Ideally, however,
we should also be able to quickly check (in polynomial 
time) whether the conditions that differ between the 
two states are likely to  make the method-body  
applicable. 
To enable re-trying a method-body that had \emph{failed}, 
we could explore 
techniques for analysing the conditions responsible for 
the failure in order to check that they no longer hold, 
as suggested in \cite{Ghallab:2016}. 
\bibliographystyle{plain}  
\bibliography{ms}
\end{document}